\documentclass{article}

\PassOptionsToPackage{numbers, sort&compress}{natbib}

\usepackage[preprint]{neurips_2025}

\usepackage[utf8]{inputenc} 
\usepackage[T1]{fontenc}    
\usepackage{hyperref}       
\usepackage{url}            
\usepackage{booktabs}       
\usepackage{amsfonts}       
\usepackage{nicefrac}       
\usepackage{microtype}      
\usepackage{xcolor}         
\usepackage{amsmath}
\usepackage{amssymb}
\usepackage{mathtools}
\usepackage{amsthm}
\usepackage{bm,bbm}
\usepackage{float}
\usepackage{balance}
\usepackage{algorithm}
\usepackage{algorithmicx}
\usepackage{algpseudocode}
\usepackage{multirow} 
\usepackage{xcolor}         
\usepackage{colortbl}
\usepackage{xspace}
\usepackage[skip=5pt]{caption}
\usepackage[capitalize,noabbrev]{cleveref}
\usepackage{placeins}

\theoremstyle{plain}
\newtheorem{theorem}{Theorem}[section]

\newtheorem{lemma}[theorem]{Lemma}
\newtheorem{corollary}[theorem]{Corollary}
\newtheorem{example}[theorem]{Example}
\theoremstyle{definition}
\newtheorem{definition}[theorem]{Definition}
\newtheorem{assumption}[theorem]{Assumption}
\theoremstyle{remark}
\newtheorem{remark}[theorem]{Remark}

\usepackage[textsize=tiny]{todonotes}

\makeatletter
\newcommand\Autoref[1]{\@first@ref#1,@}
\def\@throw@dot#1.#2@{#1}
\def\@set@refname#1{
    \edef\@tmp{\getrefbykeydefault{#1}{anchor}{}}%
    \xdef\@tmp{\expandafter\@throw@dot\@tmp.@}%
    \ltx@IfUndefined{\@tmp autorefnameplural}%
         {\def\@refname{\@nameuse{\@tmp autorefname}s}}%
         {\def\@refname{\@nameuse{\@tmp autorefnameplural}}}%
}
\def\@first@ref#1,#2{%
  \ifx#2@\autoref{#1}\let\@nextref\@gobble
  \else%
    \@set@refname{#1}
    \@refname~\ref{#1}
    \let\@nextref\@next@ref
  \fi%
  \@nextref#2%
}
\def\@next@ref#1,#2{%
   \ifx#2@ and~\ref{#1}\let\@nextref\@gobble
   \else, \ref{#1}
   \fi%
   \@nextref#2%
}

\let\oldtheequation\theequation
\renewcommand\tagform@[1]{\maketag@@@{\ignorespaces#1\unskip\@@italiccorr}}
\renewcommand\theequation{(\oldtheequation)}

\makeatother

\newcommand{\real}{\mathbb{R}}  
\def\E{\mathbb{E}} 
\def\P{\mathbb{P}} 

\newcommand\numberthis{\addtocounter{equation}{1}\tag{\theequation}}  

\renewcommand{\hat}{\widehat}  
\renewcommand{\tilde}{\widetilde}

\makeatletter
\DeclarePairedDelimiterX{\set}[1]{\{}{\}}{\,#1\,}  
\DeclarePairedDelimiterX{\norm}[1]{\lVert}{\rVert}{#1} 
\DeclarePairedDelimiterX{\bnorm}[1]{\biggl\lVert}{\biggr\rVert}{#1} 
\DeclarePairedDelimiterX{\abs}[1]{\lvert}{\rvert}{#1} 
\DeclarePairedDelimiterX{\ip}[2]{\langle}{\rangle}{#1,#2} 

\DeclareMathOperator*{\argmin}{\arg\min}
\DeclareMathOperator*{\argmax}{\arg\max}

\DeclareMathOperator*{\var}{var}  

\newcommand{\RNum}[1]{\uppercase\expandafter{\romannumeral #1\relax}} 

\newcommand*{\indep}{%
  \mathbin{%
    \mathpalette{\@indep}{}%
  }%
}
\newcommand*{\nindep}{%
  \mathbin{
    \mathpalette{\@indep}{/}%
  }%
}
\newcommand*{\@indep}[2]{%
  \sbox0{$#1\perp\m@th$}
  \sbox2{$#1=$}
  \sbox4{$#1\vcenter{}$}
  \rlap{\copy0}
  \dimen@=\dimexpr\ht2-\ht4-.2pt\relax
  \kern\dimen@
  \ifx\\#2\\%
  \else
    \hbox to \wd2{\hss$#1#2\m@th$\hss}%
    \kern-\wd2 %
  \fi
  \kern\dimen@
  \copy0 
}

\makeatother

\newcommand{\din}{d_{\texttt{in}}}

\newcommand{\defeq}{:=}
\algrenewcommand\algorithmicensure{\textbf{Output:}}

\title{Cer-Eval: Certifiable and Cost-Efficient Evaluation Framework for LLMs}

\author{%
  Ganghua~Wang \\
  Data Science Institute\\
  University of Chicago\\
  \texttt{gangahu@uchicago.edu} \\
  \And
  Zhaorun~Chen \\
  Department of Computer Science\\
  University of Chicago \\
  \texttt{zhaorun@uchicago.edu} \\
  \AND
  Bo~Li \\
  Department of Computer Science\\
  UIUC \& University of Chicago\\
  \texttt{bol@uchicago.edu} \\
  \And
  Haifeng~Xu \\
  Department of Computer Science\\
  University of Chicago \\
  \texttt{haifengxu@uchicago.edu} \\
}

\begin{document}

\maketitle

\newcommand{\ours}{Cer-Eval\xspace}

\begin{abstract}
As foundation models continue to scale, the size of trained models grows exponentially, presenting significant challenges for their evaluation. Current evaluation practices involve curating increasingly large datasets to assess the performance of large language models (LLMs). However, there is a lack of systematic analysis and guidance on determining the sufficiency of test data or selecting informative samples for evaluation.
This paper introduces a certifiable and cost-efficient evaluation framework for LLMs. Our framework adapts to different evaluation objectives and outputs confidence intervals that contain true values with high probability.
We use ``test sample complexity'' to quantify the number of test points needed for a certifiable evaluation and derive tight bounds on test sample complexity. 
 Based on the developed theory, we develop a partition-based algorithm, named \emph{\ours}, that adaptively selects test points to minimize the cost of LLM evaluation. Real-world experiments demonstrate that \ours can save 20\% to 40\% test points across various benchmarks, while maintaining an estimation error level comparable to the current evaluation process and providing a 95\% confidence guarantee.
\end{abstract}

\section{Introduction}\label{sec:intro}
    In recent years, large-language-models (LLMs) have exhibited astonishing capabilities in natural language processing. Evaluating LLMs in terms of their performance and trustworthiness therefore is crucial for understanding their strengths and limitations, guiding their development, and ensuring responsible deployment~\citep{chang2024survey}. Numerous benchmark datasets have been created to assess different aspects of LLM performance. For example, the Massive Multitask Language Understanding (MMLU) dataset~\citep{hendrycks2021measuring} evaluates the knowledge and problem-solving abilities of LLMs across multiple fields such as elementary math, law and history, identifying areas in which an LLM is inferior to humans; TrustGPT~\citep{huang2023trustgpt} is proposed to assess the potential of LLMs generating toxic or hateful contents, while PromptBench~
    \citep{zhu2023promptbench} and MMDT~\citep{xu2025mmdt} test the vulnerability of LLMs to adversarial prompts that could lead to misleading or unsafe responses.

    Despite the increasing number of benchmark datasets, little attention has been paid to the evaluation process itself. The current practice, which we call the static evaluation process, is simply reporting the average score over the entire test dataset. This is the method used by widely adopted platforms such as Gen AI on Vertex AI platform by Google, the open LLM leaderboard hosted by Huggingface, and the Evals framework by OpenAI. 
    
    However, this static evaluation approach has two major drawbacks. 
    First, it does not quantify or guarantee the reliability of the result. Here, reliability means how close the evaluation result is to the truth and how confident its conclusion is. 
    In particular, there are two sources contributing to the uncertainty in the evaluation results: the randomness in the model responses, and the randomness in the dataset used for evaluation. 
    The lack of reliability imposes difficulty in drawing a trustworthy conclusion and further tasks such as model comparison. 

    Second, it is not sample-efficient and does not adapt to various evaluation scenarios. The static evaluation has to evaluate all test points, making it expensive and time-consuming, as LLMs typically have a numerous number of parameters. However, in many cases, evaluating a subset of the dataset would suffice to reach a reliable conclusion. For example, if an LLM consistently performs poorly on a randomly selected subset of a question-answering (QA) benchmark, we can confidently conclude that this model's QA capability is below random guessing. 
    Moreover, for users who want to evaluate the model in a dynamic and evolving manner, the reliability of static evaluation process is further compromised.
    For instance, when new data points are introduced over time, we will need to re-evaluate the model to accurately reflect its performance. However, for a static process, the chance of drawing at least one wrong conclusion will approach one with repeated evaluations. 
    
    To address these challenges, we focus on two fundamental but underexplored problems in LLM evaluation: for a given LLM, test dataset, and evaluation metric,

    (P1) How to design an algorithm that adapts to different evaluation scenarios and goals and provides a certifiable guarantee for its result?

    (P2) How to strategically select test points to minimize evaluation cost while achieving a desired conclusion, and what is the minimum number of test points needed?

    To answer these questions, we propose a certifiable online evaluation process that sequentially refines evaluation results until a user-defined estimation error and confidence level is reached, e.g., the difference between the estimation and true performance is below 0.01 with 95\% probability;
    otherwise, the algorithm will notify the user that additional data points are needed for the desired estimation error and confidence level. 
    Beyond early stopping, our approach reduces evaluation costs by strategically selecting test points.
    We propose and study a concept named \emph{test sample complexity}, which quantifies the minimal number of test points needed to achieve an accurate and confident conclusion. 
    Inspired by the analysis of test sample complexity, we develop an online evaluation algorithm, \ours, which dynamically partitions the input space into regions of low variance and high probability mass. This allows the evaluation to focus on informative test points, significantly reducing the number of samples needed.
    Our contributions are summarized as follows:

    \textbf{1.} We introduce an online evaluation framework for LLMs that provides statistical guarantees on evaluation results. Unlike static evaluation, our approach applies to various evaluation goals and ensures validity within a user-specified estimation error and confidence level. 

    \textbf{2.} We propose the concept of test sample complexity, which characterizes the number of test points required to evaluate a model to a desired level. 
        Both upper and lower bounds are established for test sample complexity when the only assumption is a bounded loss function. We also show that test sample complexity can be greatly reduced when certain distributional assumptions hold. 

    \textbf{3.} Based on the developed theory, we propose \ours (outlined in \Autoref{fig:alg}), an adaptive evaluation algorithm that minimizes the evaluation cost through dynamic dataset partitioning.  
        Compared to the static evaluation baseline, \ours uses only $30\%\sim50\%$ data points and achieves a comparable evaluation accuracy in our simulation studies.
        When applied to real-world benchmarks (MMLU, AlpacaEval, and MATH) to evaluate GPT-4o, \ours achieves the same evaluation accuracy using only $60\%\sim80\%$ of the test data.

    The rest of paper is organized as follows. \Autoref{sec:lit} introduces the related literature. \Autoref{sec:formulation} formulates the problem setup and defines test sample complexity. \Autoref{sec:general} provides test sample complexity bounds for general cases, while \Autoref{sec:var} presents how to improve those bounds when distributional assumptions on the model and task are met, and proposes adaptive evaluation algorithms motivated by the developed theory. \Autoref{sec:exp} conducts extensive experiments on the proposed algorithm compared to the baseline. We conclude the paper in \Autoref{sec:con}.

\begin{figure*}[t]
    \centering
    \includegraphics[width=\linewidth]{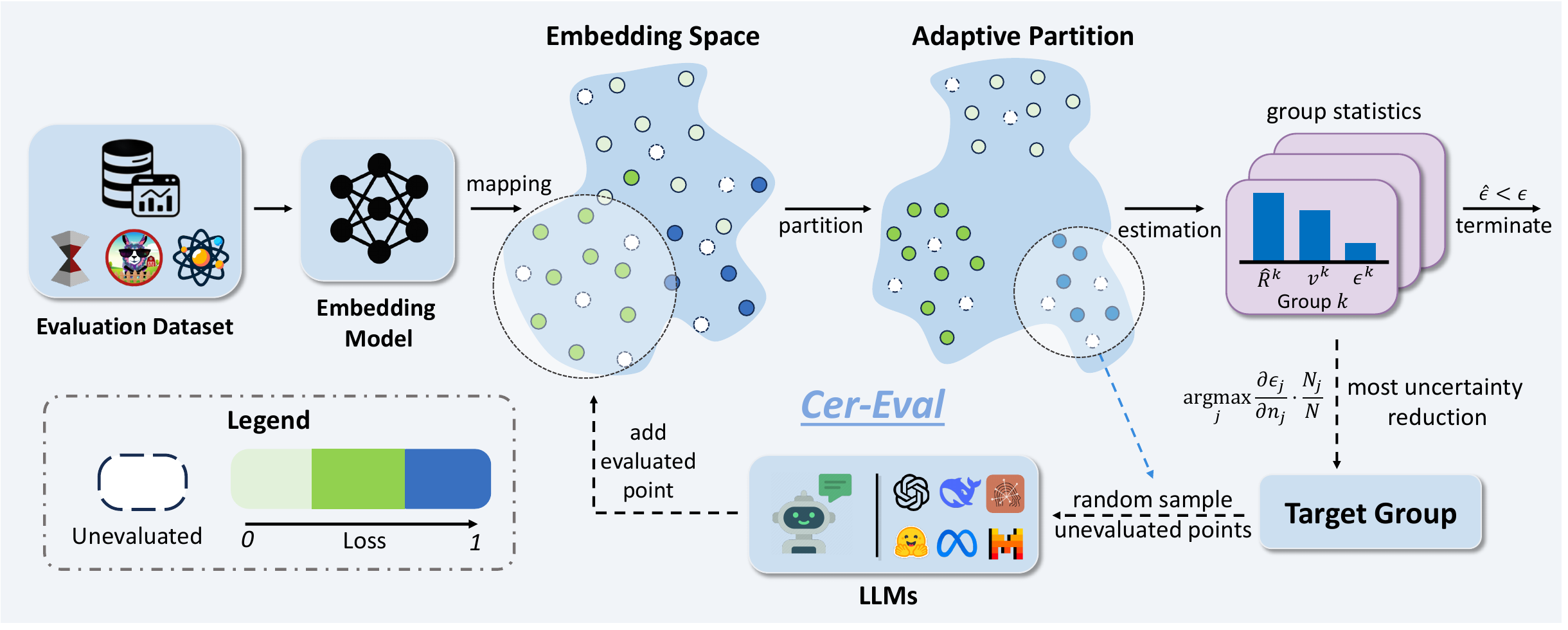}
    \caption{Overview of \ours, a partition-based adaptive evaluation algorithm. It iterates through four steps until the termination condition is met: (1) partition the dataset based on the evaluated points to minimize evaluation uncertainty, (2) compute summary statistics for each partition, (3) identify the partition that reduces uncertainty the most, and (4) sample and evaluate a new test point from the selected partition.}
    \label{fig:alg}
    \vspace{-0.2in}
\end{figure*}
    
\section{Related Work}\label{sec:lit}

\textbf{LLM evaluation.} 
    Existing literature of LLM evaluation primarily focuses on (1) discussing what aspects of LLM capability should be evaluated~\citep{liu2023trustworthy,chang2024survey,gao2024llm} and (2) proposing appropriate datasets and criterion to assess LLM performance~\citep{hendrycks2021measuring, lin-etal-2022-truthfulqa, chiang2024chatbot, zhang-etal-2024-safetybench, dubois2024length, chen2025shieldagent, chen2024mj, chen2024safewatch}. 
    However, relatively little attention has been given to the evaluation process itself. The standard practice for LLM evaluation remains simple: computing the average score over a test dataset based on a selected evaluation metric. Recently, \citep{miller2024adding} proposed to add error bars to quantify evaluation uncertainty, and \citep{chiang2024chatbot} constructs an approximate confidence level using bootstrapping. Nevertheless, these methods are empirical and lack valid statistical guarantees in finite-sample scenarios. As a result, the reliability of current evaluation practices is not formally ensured. Addressing this gap is one of the key contributions of our work.

\textbf{Efficient evaluation.}
    Researchers have recognized that evaluating LLMs on the full dataset can be computationally expensive. To mitigate this, \citep{polo2025tiny, kipnis2024texttt, xu2024data} proposed to choose a representative subset of data points to approximate the full evaluation result; \citep{boyeau2024autoeval, fisch2024stratified} adopted a stratified sampling technique to improve evaluation accuracy; and \citep{zhang2024collaborative} proposed to predict model performance using historical performance trends of similar models and tasks.
    However, these approaches also lack formal guarantees on their evaluation results. 
    Additionally, they focus solely on the test dataset without leveraging the properties of the model being evaluated. In contrast, our proposed method is adaptive to both the test dataset and the model, leading to efficient evaluations with guarantees.

\textbf{Sequential hypothesis testing.}
    Our proposed evaluation process involves a sequential selection of test points and decision making, and we certify our evaluation results by constructing a sequence of confidence intervals (CI) covering the truth with high probability. 
    The problem of constructing valid sequential CIs has been studied in the literature~\citep{farrell1964asymptotic, karp2007noisy,zhao2016adaptive}, with key techniques relying on classical Hoeffding-type concentration inequalities and a union-bound argument. 
    However, existing methods do not incorporate model- and dataset-specific structure, which we find crucial for improving evaluation efficiency.
    By incorporating a Bernstein-type concentration inequality and a partition-based approach, we prove that the needed test sample size can be greatly reduced under certain conditions.

\vspace{-0.1in}

\section{Problem Formulation}\label{sec:formulation}
    Given a trained LLM $f: X \in \mathcal{X} \to Y \in \mathcal{Y}$, we aim to evaluate its performance on a given task. We assume that the input space $\mathcal{X} = \real^{\din}$ consists of the embedding vector of tokens with maximum length $\din$, while the output space $\mathcal{Y}$ is general. The model $f$ can be either deterministic or non-deterministic. 
    Throughout this paper, we consider a fixed evaluation task characterized by a joint distribution $P_{XY}$ over $(X, Y)$. 
    The model performance in task $P_{XY}$ is quantified by its prediction error $R(f, P_{XY}) \defeq \E_{(X, Y) \sim P_{XY}} \ell(f(X),Y)$ for some loss function $\ell$. Examples of $\ell$ include the zero-one loss for multiple-choice tasks, similarity-based metrics for natural language understanding tasks, or human- or LLM-based scores for reasoning tasks. For notational simplicity, we denote $R(f, P_{XY})$ as $R$ when there is no ambiguity.  

    In practice, model performance is often assessed on a given test dataset $D_n=\{(X_i,Y_i), i=1,\dots,n\}$. In this case, we assume that $D_n$ represents the underlying task distribution $P_{XY}$. That is, data points in $D_n$ are independently and identically (IID) drawn from $P_{XY}$. A concrete example of LLM evaluation is provided below. 

    \begin{example}
        Massive Multitask Language Understanding (MMLU) dataset~\citep{hendrycks2021measuring} includes more than 15,000 multiple choice questions covering 57 subfields. Each question has four answer choices, with only one being correct. Researchers evaluate LLMs' natural language understanding capability on this dataset by averaging the prediction accuracy across all subfields. In this case, the output space is $\{1,2,3,4\}$, and a zero-one loss is used as the evaluation metric. It is found that a basline human achieves an average accuracy of 34.5\%, and many LLMs perform near-randomly.  
    \end{example}

    \textbf{Evaluation Goal and Process.} 
    The goal of model evaluation is to obtain an \textbf{accurate} and \textbf{confident} estimation of $R$. Here, accuracy means how close the estimation is to the truth, and confidence means the probability that our claim is correct. The user can specify a desired confidence level $1-\delta$ ($\delta$ is also known as the failure probability) and estimation error level $\epsilon$. Two common evaluation goals are:
    
    (1) Estimate $R$ within an error level of $\epsilon=0.01$ with $95\%$ confidence. 
    
    (2) Determine whether $R$ exceeds a threshold (e.g., 0.5) with $95\%$ confidence. 
    
    Therefore, a certifiable evaluation algorithm should provide an estimate of $R$ along with a confidence interval (CI) of radius $\epsilon$, which is guaranteed to contain the true error $R$ with probability at least $1-\delta$. Notably, the second goal above is equivalent to obtaining an estimate with $\epsilon$ implicitly determined by $R$ and the threshold. Moreover, users may wish to dynamically adjust $\epsilon$ during the evaluation process. 
    The current practice, which we call a static evaluation process, requires evaluating all test data at once, described below. 
        \begin{definition}[Static evaluation process]
        A static evaluation process tests all data points in a test dataset $D_n$, and output a single estimated prediction error and CI. 
    \end{definition}
    
    Static evaluation
    is unable to handle an implicitly defined or flexible $\epsilon$. To address this issue, we propose an online evaluation framework, as defined below. 
    \newcommand{\nalg}{N}
    \begin{definition}[Online evaluation process]
        An online evaluation algorithm $\mathcal{A}$ sequentially selects test points until the desired evaluation error level $\epsilon$ is achieved, or all available test data points are used. 
        The number of \textit{evaluated} test points at termination is denoted as $\nalg$.
    \end{definition}

    \begin{definition}[$(n,\epsilon,\delta)$-certified evaluation algorithm]
        An algorithm $\mathcal{A}$ is called $(n,\epsilon,\delta)$-certified if 
        \begin{align*}
            \P(\nalg \geq n \text{, or $\mathcal{A}$ produce at least one CI}  \text{ that does not contain truth} ) \leq \delta.
        \end{align*}
    \end{definition}

    \begin{remark}[Practical meaning of certified algorithms]
    For example, when the user's goal is estimating the model performance $R$, an estimation obtained from an $(n,\epsilon,\delta)$-certified algorithm is guaranteed to be within $\epsilon$ of the true value with probability at least $1-\delta$.
    \end{remark}
    
    \newcommand{\ntest}{n^*}

   \begin{definition}[Test sample complexity]
      Consider any algorithm $\mathcal{A}$ that is $(n,\epsilon,\delta)$-certified for evaluating a model $f$ on task $P_{XY}$. 
      Test sample complexity $\ntest \defeq n(\epsilon,\delta,f,P_{XY})$ is the smallest $n$ over all possible choices of algorithms $\mathcal{A}$.   
    \end{definition}  

    

    A $(n,\epsilon,\delta)$-certified test algorithm provides guarantees on how confident the evaluation result is and how many test points are needed for this algorithm. Test sample complexity, which is the minimal required number of test points to draw a confident and accurate conclusion, further characterizes the fundamental difficulty of evaluating a model on a given task. 
    The subsequent sections are dedicated to obtain bounds on the test sample complexity, and propose efficient evaluation algorithms, thereby answering the core research problems introduced in \Autoref{sec:intro}. 



    
    \begin{remark}
        The static evaluation process can be regarded as a special online algorithm that only yields a result after evaluating all $n$ points. 
    \end{remark}

    \vspace{-0.1in}

\section{Intrinsic Limits of Test Sample Complexity}\label{sec:general}
    In this section, we establish matching upper and lower bounds on test sample complexity, assuming only that the loss function is bounded. \Autoref{thm:no} indicate the fundamental limit on the number of test points needed for a certifiable evaluation across general models and tasks. 
    
    \begin{assumption}[Bounded loss]\label{asmp:bound}
        We assume that the loss function $\ell$ is bounded. Without loss of generality, let $0 \leq \ell(f(X), Y) \leq 1.$
    \end{assumption}

    \newcommand{\diff}{\epsilon}
    \begin{theorem}\label{thm:no}
        Let $\epsilon>0$ be the desired estimation error level and $0<\delta<1$ be the failure probability. 
        Under Assumption~\ref{asmp:bound}, we have the following results:
        \begin{itemize}
        \item (Upper Bound) There exists an $(n,\epsilon,\delta)$-certified online evaluation process (Algorithm~\ref{alg:seq} in Appendix C)  with
        $$n \leq  O\biggl(\frac{\ln(1/\delta)+\ln\ln(1/\diff)}{\diff^{2}} \biggr).$$
        \item (Matching Lower Bound) For any function $n'(\epsilon,\delta,f,P_{XY})$ such that  
        \begin{align*}
        \lim_{\diff,\delta \to 0} n'(\epsilon,\delta,f,P_{XY})  \frac{\diff^{2}}{\ln(1/\delta)+\ln\ln(1/\diff)} > 0,
    \end{align*}
    no algorithm can be $(n', \epsilon, \delta)$-certified for all sufficiently small $\diff$ and $\delta$.
        \end{itemize}
    \end{theorem}

    \vspace{-0.1in}
    The bounds in Theorem~\ref{thm:no} depend on two key parameters, the estimation error level $\diff$ and failure probability $\delta$.
    Clearly, a smaller $\delta$ requires a higher level of confidence of the evaluation, thus a larger test sample complexity is needed.
    A smaller $\diff$ demands a greater evaluation accuracy, leading to a larger test sample complexity. 
    Compared to training sample complexity, test sample complexity focuses on a specific model and task instead of learning from a function class. Furthermore, in real-world evaluations, $\epsilon$ may be implicitly determined or dynamically adjusted.
    This necessitates a sequential evaluation and therefore introduces a sequential decision-making challenge, requiring the additional iterated logarithm term $\ln\ln(1/\diff)$ to control the overall failure probability -- an effect absent in classical training sample complexity bounds.
    
    \newcommand{\alggen}{\mathcal{A}_{\text{seq}}}
    Notably, even the vanilla online evaluation process (\Autoref{alg:seq}) can significantly reduce the amount of test points compared to the static evaluation process, particularly when the desired estimation error level $\epsilon$ is not too small. 
    
\vspace{-0.1in}

\section{Sample-Efficient Evaluation via Partition}\label{sec:var}
    In this section, we go beyond the intrinsic statistical bounds by incorporating additional knowledge about the model and task. 
    To further save test points, the key idea is to pay more attention to areas with higher uncertainty, instead of drawing test points IID from the entire space. Two critical observations drive this approach: (1) 
    An area with smaller loss variance is less uncertain and requires fewer test points for confident evaluation; and (2) 
    Properly dividing the input space may lower the variance within each partition, reducing the number of test points needed for evaluation.
    Thus, if we can divide the task distribution $P_{XY}$ into $K$ disjoint areas and reduce variance within each, we can achieve a more sample-efficient evaluation than the general approach in \Autoref{sec:general}. We formulate this idea in \Autoref{thm:prune}.


        \begin{definition}[Benign Partition]\label{asmp:loss_dist}
        Consider any partition $\{A_k\}_{k=1,\dots,K}$ on the support of $P_{XY}$, and  $v_k = \var\{\ell(f(X),Y) \mid A_k\}$ be the variance of the loss conditioned on $A_k$. 
        Given a test dataset $D_n=\{(X_i,Y_i),i=1,\dots,n\}$, let   $\tilde{D}_k= A_k \cap D_n$ and $n_k \defeq \abs{\tilde{D}_k}$. 
        We say $\{\tilde{D}_k\}_{k=1,\dots,K}$ is a benign partition of $D_n$ if the following holds:
        \begin{align*}
            n_k/n \geq \ln(K+1) \max\{v_k,\diff^{2/3}\}, k=1,\dots,K.
        \end{align*}
    \end{definition}

    \begin{theorem}\label{thm:prune}
        Suppose Assumptions~\ref{asmp:bound} holds, and let $n= \Theta(\diff^{-2}\{\ln(1/\delta)+\ln\ln(1/\diff)\} )$ denote the tight bound of Theorem \ref{thm:no} for some $\epsilon, \delta$. 
        Then \Autoref{alg:adap_var} operating with a benign partition of $D_n$ (Definition \ref{asmp:loss_dist}) is $(n', \epsilon, \delta)$-certified, such that
        $$\rho \defeq \frac{n'}{n} = O\biggl(\ln(K+1) \sum_{k=1}^K \max\{v_k, \diff^{2/3}\}\biggr).$$
    \end{theorem}
\vspace{-0.1in}

    \Autoref{thm:prune} suggests that, with a benign partition, we can use only $\rho$ percent of test points for evaluation, compared to the vanilla evaluation algorithm in \Autoref{sec:general} where no additional knowledge is available. Furthermore, as $\diff \to 0$, allowing $K$ to grow can lead to $\rho \to 0$, significantly reducing the number of test points needed.

\textbf{Illustration via examples. }  Benign partitions exist for a wide range of models and tasks, and we can have a good estimate of the saving ratio $\rho$. We illustrate it by the following corollary and example. 

    \begin{corollary}[Super-Gaussian loss distribution]\label{coro:superG}
        Suppose the loss distribution satisfies $h(z) \defeq \P(\ell(f(X),Y) = z) \geq A\exp\{-z^2/\sigma^2\}$ for $z \in [0,1]$ with constants $A,\sigma^2 >0$, then, we have 
        $\rho = O(\ln(K+1)/K)$ for any $K$ such that $K/\ln(K+1) \geq \exp(1/\sigma^2)/A$ and $K \leq \diff^{-1/3}$. As a result, $\rho = O(\diff^{1/3}\ln(1/\diff))$ when $K=O(\diff^{-1/3})$.
    \end{corollary}

    \begin{example}[Equally distributed problem difficulty]
        Consider a dataset split into $K$ difficulty levels, each having an equal number of test points. Also, the prediction loss of $f$ in $A_k$ lies in $[(k-1)/K, k/K]$ uniformly.
        For example, a dataset assessing LLMs on math problems may contain questions collected from primary school, high school, undergraduate, and graduate levels. Since the loss distribution is uniform on $[0,1]$, Corollary~\ref{coro:superG} applies, yielding $\rho = O(\ln(K+1)/K)$. 
    \end{example}

    \begin{remark}
        The term $\max\{v_k, \diff^{2/3}\}$ in \Autoref{thm:prune} arises from estimating the unknown loss variance in each partition. We can improve this term to $\max\{v_k, \diff\}$ if an upper bound asymptotically equivalent to $v_k$ is known. Moreover, equally-space partition is a special partition that improves this term to $1/K^2$, leading to $\rho = O(\diff\ln(1/\diff))$ in the super-Gaussian example above.
    \end{remark}

    \begin{remark}
        Our developed theory provides support for similarity-based dataset pruning methods, such as clustering. Those methods assume that the model performance is similar in a small neighborhood of any point $(X, Y)$. If the loss function is continuous, then the performance at $(X,Y)$ suffices to approximate local performance on an $\epsilon$ ball $B_{\epsilon}(X,Y) \defeq \{(X',Y'): \norm{(X,Y), (X',Y')} \leq \epsilon\}$ around it, reducing the number of required test points. 
    \end{remark}

    \vspace{-0.05in}

    \textbf{Finding effective benign partitions.}  As shown in \Autoref{thm:prune}, finding a benign partition that simultaneously minimizes in-group variance and maximizes the probability mass of each group is critical in enhancing test efficiency. 
    
    However, such partition information is usually unavailable in practice. To address this concern, we design \ours (\Autoref{fig:alg}), which dynamically partitions the input space in a model- and data-driven manner to maximize the benefit brought by the benign partition. 
    \ours repeats the following steps until the desired estimation error level is achieved or all points in $D_n$ have been evaluated, with full details in \Autoref{alg:adap_var}: 
    (1) Adaptive partition. Partition the input space based on the evaluated points by minimizing the uncertainty level. We propose to adopt $1$-nearest neighbor algorithm in the partition subroutine, as detailed in \Autoref{alg:part}.
    (2) Estimation. Compute the sample mean and associated CI radius for each group.
    (3) Target group selection. Identify the group that contributes most to uncertainty.
    (4) New sample evaluation. Sample and evaluate a new test point from the target group.

    \begin{algorithm*}[tb]
        \caption{Certified Evaluation with Adaptive 
 Partition (\ours)}\label{alg:adap_var}
        \begin{algorithmic}[1]
            \Require The estimation error level $\epsilon$, the failure probability $\delta$, test dataset $D$, warm start steps $m$, and a partition subroutine.
            \State Select the first $m$ points $S=\{(X_i, Y_i), i=1,\dots,m\}$ from $D$  \Comment{\textbf{Step 0}: Warm-up sampling}
            \State Evaluate $Z_i = \ell(f(X_i), Y_i), i=1,\dots,m$
            \While{True}
                \State Partition $D$ to $K$ areas $\tilde{D}_1,\dots,\tilde{D}_K$ by the partition subroutine. \Comment{\textbf{Step 1}: Partition dataset}
                \For{$k=1,\dots,K$} \Comment{\textbf{Step 2}: Calculate summary statistics per group}
                    \State Let $S_k \leftarrow \{(X,Y) \in S: (X, Y) \in \tilde{D}_k\}$
                    \State Let $n_k \leftarrow \abs{S_k}$, $N_k \leftarrow \abs{\tilde{D}_k}$ \Comment{Sample size of group $k$}
                \State Let $\hat{R}^k \leftarrow n_k^{-1}\sum_{i=1}^{n_k} Z_i^k$ \Comment{Empirical mean in $S_k$}
                \State Let $v_k \leftarrow n_k^{-1} \sum_{i=1}^{n_k}(Z_i^k-\hat{R}_k)^2 $ \Comment{Empirical variance}
                \State Let $\eta_k \leftarrow \sqrt{\{2\ln(\log(n_k)+1)+\ln(16K/\delta)\}/n_k}$ 
                \State Let $\epsilon_{k} \leftarrow 2\eta_k^2/3+2\sqrt{(v_k+\eta_k+\eta_k^2)\eta_k^2}$ \Comment{Confidence interval radius}
                \EndFor
                   

                        \State Let $\hat{R} \leftarrow \sum_{k=1}^{K} N_k\hat{R}^k/N$, $\hat{\epsilon} = \sum_j N_j \epsilon_j/N$ \Comment{Performance estimate and CI raduis}
                        \If{$\hat{\epsilon} \leq \epsilon$} \Comment{\textbf{Step 4}: Termination condition}
                            \State Terminate and return $\hat{R}, \hat{\epsilon}$
                    \EndIf
                    \State Terminate if all points in $D$ are evaluated
                     \State $k\leftarrow \argmax_{1\leq j\leq K, n_j \leq \abs{D_j}} \frac{\partial \epsilon_j}{\partial n_j} \cdot \frac{N_j}{N}$ \Comment{\textbf{Step 3}: Target sampling, identify group that contributes most uncertainty} 
                    \State Select a data point $(X_j, Y_j)$ from $\tilde{D}_k \backslash S_k$
                    \State Let $n_{k} \leftarrow n_{k} + 1$, add $(X_j, Y_j)$ to $S$ and evaluate $Z_{n_{k}}^{k} \leftarrow \ell(f(X_j), Y_j)$
            \EndWhile
            \Ensure The estimated loss $\hat{R}$, confidence interval radius $\hat{\epsilon}$, and number of evaluated points $\sum_k n_k$
        \end{algorithmic}
    \end{algorithm*}    

\vspace{-0.1in}

     \begin{algorithm}[tb]
        \caption{Subroutine: Partition by $1$-nearest neighbor}\label{alg:part}
        \begin{algorithmic}[1]
            \Require The test dataset $D$, evaluated points $S=\{(X_i,Y_i),i=1,\dots\}$ and the corresponding loss values $\{Z_i,i=1,\dots\}$.
            \For{$k=1,\dots,\lceil\ln(\abs{S})\rceil+1$}
                \State Assign $i$-th data point $(X_i,Y_i)$ in $S$ a label $\lfloor kZ_i \rfloor$, $i=1,\dots,n$
                \State Train a $1$-nearest neighbor classifier $\mathcal{C}_k$ on a random subset of S.
                \State Partition $D$ by the labels predicted using $\mathcal{C}_k$.
                \For{$j=1,\dots,k$}
                    \State Calculate $\epsilon_j$ following lines 6-11 in \Autoref{alg:adap_var} 
                \EndFor
                \State $\tilde{\epsilon}_k \leftarrow \sum_j N_j\epsilon_j/N$
            \EndFor
            \Ensure $K \leftarrow \argmin_k \tilde{\epsilon}_k$, $\tilde{D}_k$'s partitioned by $\mathcal{C}_K$
        \end{algorithmic}        
    \end{algorithm} 
    \vspace{-0.05in}
    
\section{Experiment}\label{sec:exp}
\vspace{-0.05in}

\subsection{Simulation}
\vspace{-0.05in}

    \textbf{Synthetic Data.} We generate simulated datasets as follows. First, we choose $K$ as the number of partitions. For each partition $A_k, k=1,\dots,K$, the query and response pairs follow a joint Gaussian distribution $(X,Y) \sim N(c_k, \sigma^2 I)$, where $I$ is the identity matrix and $c_k=(\lambda k,0,0,\dots) \in \real^d$ with some constants $d$, $\lambda$ and $\sigma^2$. Moreover, the loss $\ell(f(X),Y) \mid A_k$ follows a truncated Gaussian distribution with mean $(k-1/2)/K$ and variance $1/K^2$. In particular, we consider the following three scenarios:
    (S1) Single partition: $K=1$.
    (S2) Multiple easy-to-distinguish groups: $K=3$, $\lambda=5$.
    (S3) Multiple hard-to-distinguish groups: $K=3$, $\lambda=1$.
    For all scenarios, we set the failure probability $\delta=0.05$, input dimension $d=10$, variance $\sigma^2=1$, and generate a test dataset of size $n=5000$. 
    
    \textbf{Evaluation algorithms.} We compare three proposed online evaluation algorithms against a static evaluation baseline:
    (1) \textbf{Base}: The static evaluation process that evaluates all data points and provide a confidence interval. (2) \textbf{Seq}: The vanilla online evaluation process, detailed in \Autoref{alg:seq}. (3) \textbf{\ours}: Our proposed adaptive evaluation algorithm, as described in \Autoref{alg:adap_var}. (4) \textbf{Oracle}: A special case of \ours that uses the true partition as the partition subroutine, which is theoretically optimal.
    %

    \begin{figure*}[t]
            \centering
                \begin{minipage}{0.32\linewidth}
                    \centering                        \includegraphics[width=\linewidth]{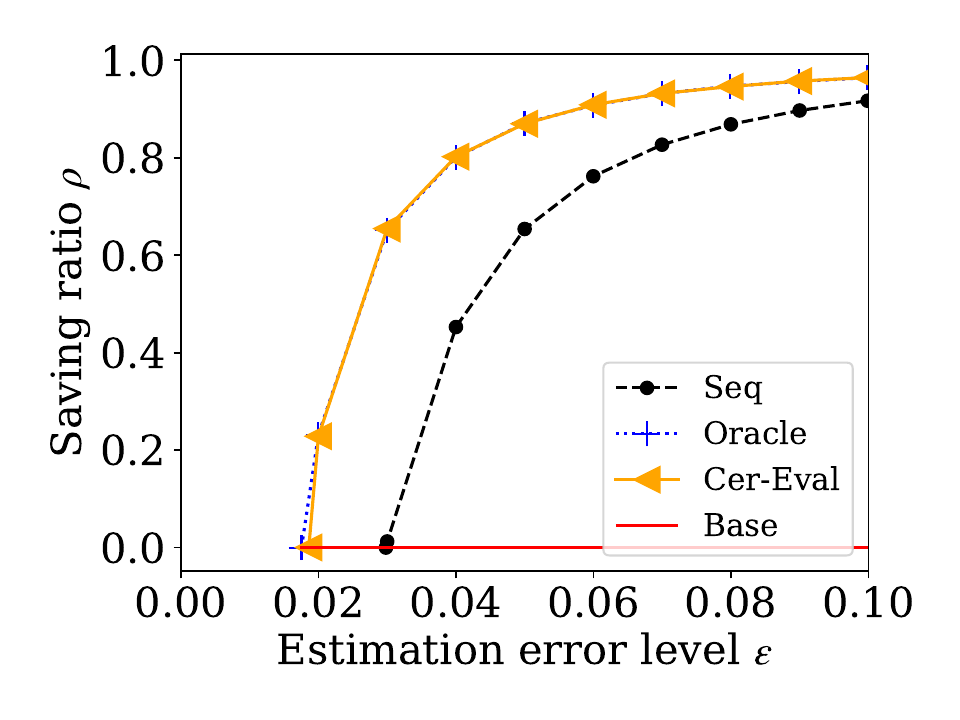}
                \end{minipage}
                \hfill
                \begin{minipage}{0.32\linewidth}
                    \centering
                        \includegraphics[width=\linewidth]{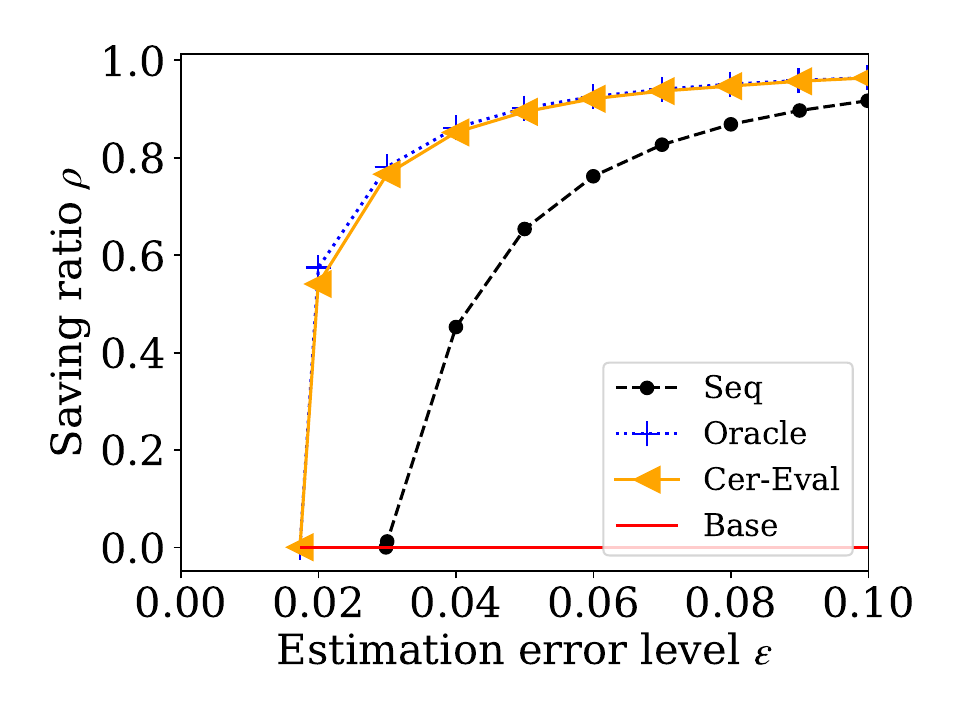}
                \end{minipage}
                \hfill
                \begin{minipage}{0.32\linewidth}
                    \centering
                        \includegraphics[width=\linewidth]{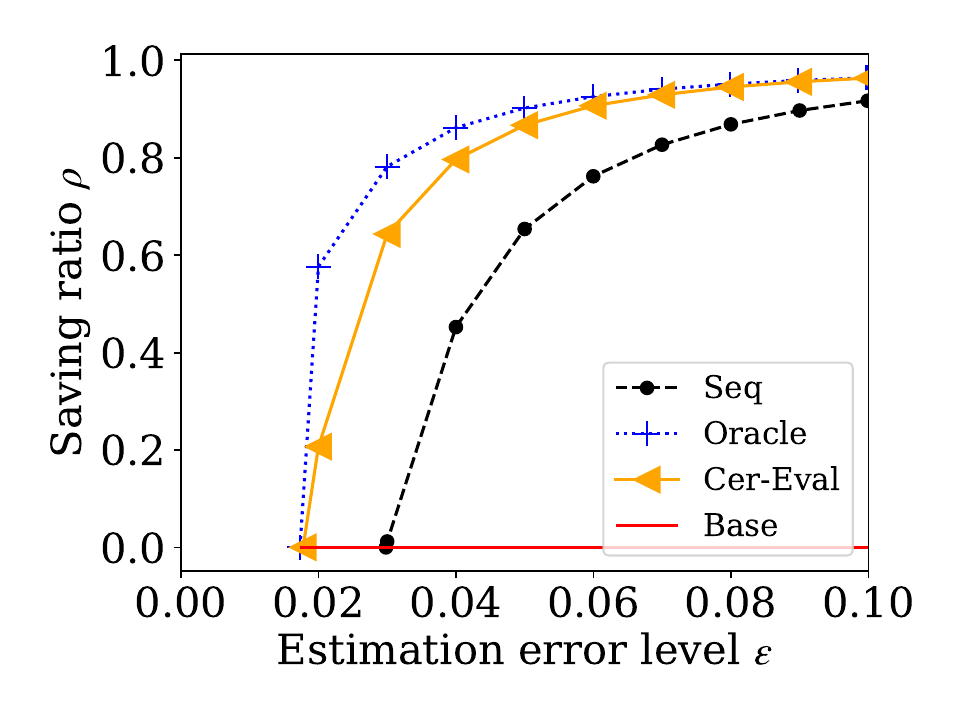}
                \end{minipage}
                 \caption{Percentage of  test points saved by \ours compared to the baselines on  \textbf{Synthetic Data} under (\textbf{Left}) the single partition scenario, (\textbf{Middle}) the easy-to-distinguish scenario, and (\textbf{Right}) the hard-to-distinguish scenario.}
       \label{fig:sim}
    \end{figure*}

    \begin{figure*}[tb]
    \vspace{-0.15in}
            \centering
                \begin{minipage}{0.32\linewidth}
                    \centering                        \includegraphics[width=\linewidth]{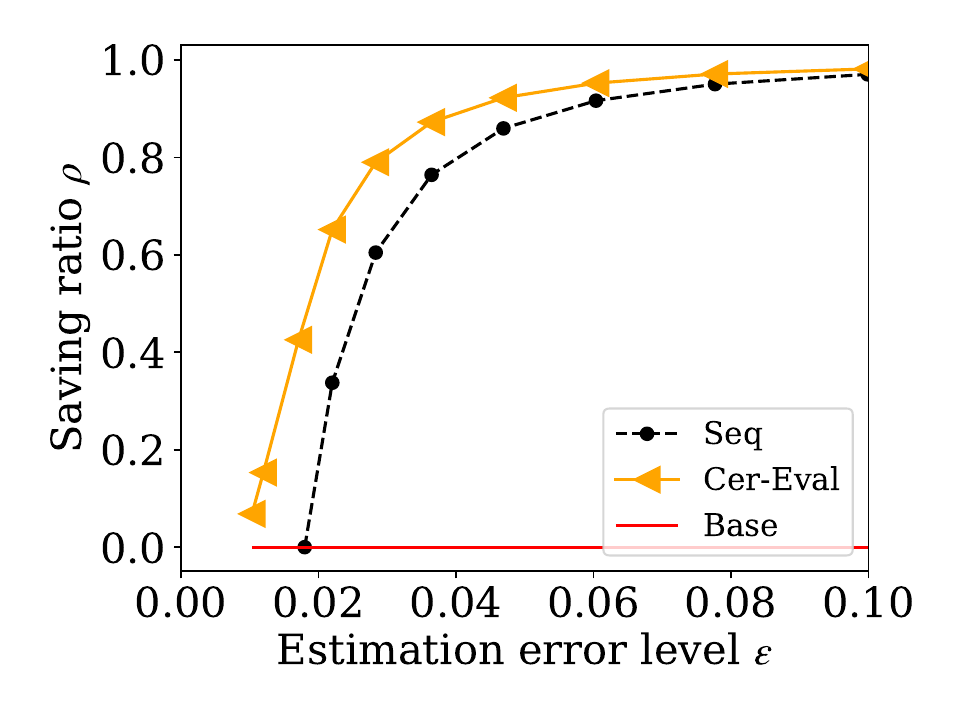}
                \end{minipage}
                \hfill
                \begin{minipage}{0.32\linewidth}
                    \centering
                        \includegraphics[width=\linewidth]{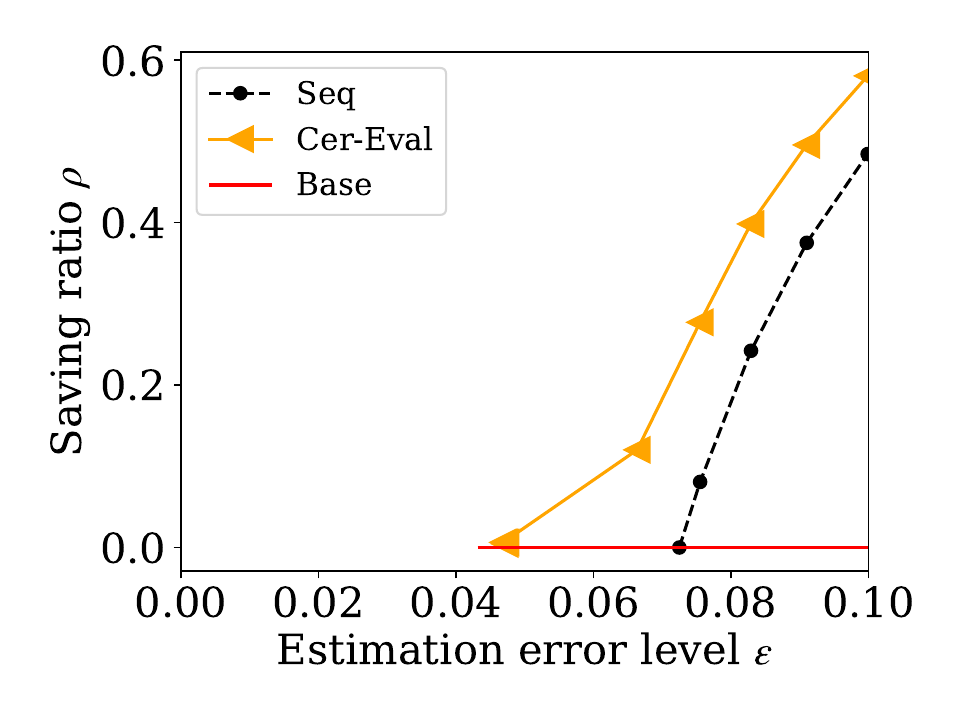}
                \end{minipage}
                \hfill
                \begin{minipage}{0.32\linewidth}
                    \centering
                        \includegraphics[width=\linewidth]{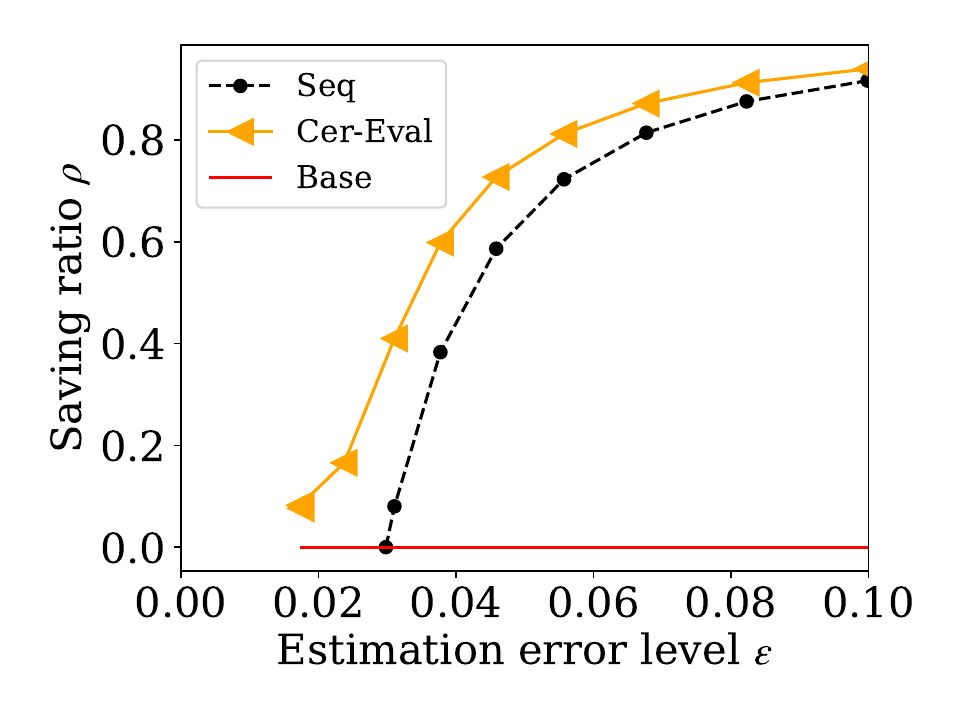}
                \end{minipage}
                 \caption{Percentage of  test points saved  by \ours compared to   baselines in \textbf{Real-World Benchmarks under GPT-4o}:  (\textbf{Left}) the MMLU dataset, (\textbf{Middle}) the AlpacaEval dataset, and (\textbf{Right}) the MATH dataset.}
       \label{fig:4o}
        \vspace{-0.2in}
    \end{figure*}

   \textbf{Evaluation metric.} 
   We evaluate the average loss for various values of estimation error level $\epsilon \in [\epsilon^*, 0.1]$, where $\epsilon^* = \sqrt{\log(1/\delta)/(2n)}$ is the estimation error level achieved by Base.
   For each test algorithm and estimation error level, we 
   report (1) the average saving ratio $\rho$, indicating the proportion of test points saved compared to `Base',
   and (2) the empirical failure probability, measuring the frequency at which $R$  falls outside the computed confidence interval.
   The experiment is replicated 20 times in each scenario.

   \textbf{Findings.} The saving ratios for all scenarios are reported in \Autoref{fig:sim}, while the empirical failure probability is zero across all methods. 
   Each $(\rho,\epsilon)$ pair can be regarded as a feasible solution, and a pair closer to the upper-left corner is preferred. The area under the $\rho$-$\epsilon$ curve therefore reflects the evaluation efficiency of an algorithm. 
   Specifically, we have the following key observations:
   \begin{itemize}
    \item Algorithms performing variance reduction partitions significantly improve test efficiency.
    In the easy-to-distinguish scenario (\Autoref{fig:sim}, middle panel), \ours and Oracle save nearly $60\%$ of test points compared to Base for $\epsilon=0.02$.
    \vspace{-0.05in}
    \item Partition-based algorithms improve efficiency even for a single group. 
    When $K=1$, both Oracle and \ours save about $20\%$ of test points for $\epsilon=0.02$ and $70\%$ for $\epsilon=0.03$, compared to Base. 
    \vspace{-0.05in}
    \item Partition quality is crucial. \ours evaluates more efficiently in the easy-to-distinguish scenario than in the hard one, highlighting the importance of effective partitioning. Oracle is \ours with the knowledge of perfect partitions, achieving even better performance.
    \vspace{-0.05in}
    \item Partition-based algorithms outperforms Seq in all scenarios. 
    This is because they effectively utilize model- and dataset-specific information, while Seq does not take those information into account. 
    \vspace{-0.05in}
    \item All online evaluating algorithms guarantee the desired confidence level, successfully including the truth in the reported CI with high probability.
   \end{itemize}
   \vspace{-0.1in}
   In short, simulation experiments confirm that \ours can greatly save the needed test sample points by adapting to the model and dataset of interest, with a controlled failure probability on the evaluation result.

\vspace{-0.1in}
\subsection{Real-World Benchmarks}
    \textbf{Datasets.} We conduct experiments on the following three real-world datasets: (1) MMLU~\citep{hendrycks2021measuring}: This dataset assesses an LLM's knowledge on 14,042 multiple choice questions across 57 subjects, such as history and math. Zero-one loss is used as the evaluation metric and we are interested in the model accuracy. (2) AlpacaEval~\citep{dubois2024length}: An automated evaluation benchmark that evaluates the LLM's natural language generating capability. We focus on evaluating the win rate of a target model's generated text compared to a reference model. A voting probability (or win rate score) is used as the evaluation metric. (3) MATH~\citep{hendrycksmath2021}: A dataset used to measure LLMs' math problem solving abilities. We use the zero-one loss to evaluate the model accuracy. 


    \textbf{Algorithms, models, and results.} 
    Since real-world datasets lack a true partition, we only compare three methods: Base, Seq, and \ours. 
    The embedding vectors for \ours are obtained using a pre-trained BERT model~\citep{Devlin2019BERTPO}, with an ablation study on embedding models provided in Appendix D.
    The empirical failure probability is calculated as the proportion of trials where the CI does not contain the model's average performance across the entire dataset. 
    Other experimental settings follow those of the simulation study.

    We assess four models across all datasets: GPT-4o, Llama3 8B, Mistral 7B, and Qwen2 7B. \Autoref{fig:4o} shows the experimental result for GPT-4o. Curves for other models are similar, hence are deferred to Appendix D together with full experiment details. 
    The empirical failure probability remains zero for all methods.
    We find that:
    \begin{itemize}
        \item Adaptive partition algorithm improves evaluation efficiency but varies by datasets and models. Aligning with the simulation study, the partition quality is crucial for \ours. As the partition is found adaptive to each model and task, the saving ratio of \ours thus varies.  
        On the MATH and MMLU datasets, \ours reduces the required test samples by $30\%\sim40\%$ for $\epsilon=1.5\epsilon^*$ and even $5\%\sim10\%$ for $\epsilon=\epsilon^*$. However, it achieves lower savings on AlpacaEval, reducing test points by only $10\%$ at $\epsilon=1.5\epsilon^*$.
        \vspace{-0.05in}
        \item \ours consistently outperforms `Seq'. By leveraging variance information, \ours uses fewer test points than Seq and successfully identifies meaningful partitions.
        Notably, on both AlpacaEval and MATH datasets, there is a change point of the $\rho$-$\epsilon$ curve for \ours. 
        For example, on AlpacaEval, the curve flattens when $\epsilon<0.067$. A closer investigation reveals that \ours has detected two distinct data groups after this point, where the model performs well on one group and performs poorly on the other. 
        It leads to significantly reduced within-group variance, therefore \ours obtains a confident evaluation result with fewer points. This observation aligns with our theoretical findings.
         \vspace{-0.05in}
        \item \ours helps determine the sufficiency of test data. Note that for MMLU and MATH dataset, \ours do not evaluate all data points to achieve an estimation error level of $\epsilon^*$. It indicates that these two datasets already have sufficient data for even a smaller error level or higher confidence level. In contrast, AlpacaEval dataset has to collect more test points for a more accurate evaluation. 
         \vspace{-0.05in}
        \item All methods maintain the desired confidence level. Across all model-dataset combinations, the empirical failure probabilities remains below $0.05$, corroborating the reliability of the proposed algorithms.    
    \end{itemize}

        \begin{figure*}[t]
            \centering
                \begin{minipage}{0.49\linewidth}
                    \centering                        \includegraphics[width=\linewidth]{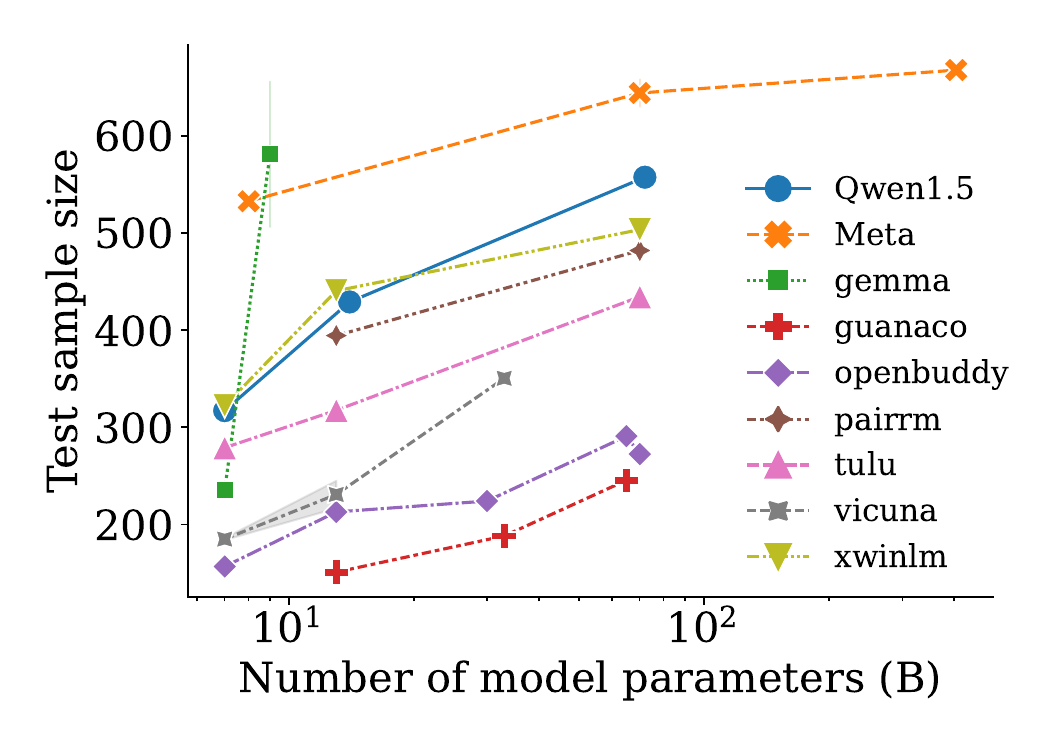}
                    \caption{Number of needed test points v.s. model size when evaluating models from multiple families using \ours, with estimation error level $\epsilon=0.07$ and failure probability $\delta=0.05$.}
        \label{fig:size}
                \end{minipage}
                \hfill
                \begin{minipage}{0.49\linewidth}
                    \centering
                        \includegraphics[width=\linewidth]{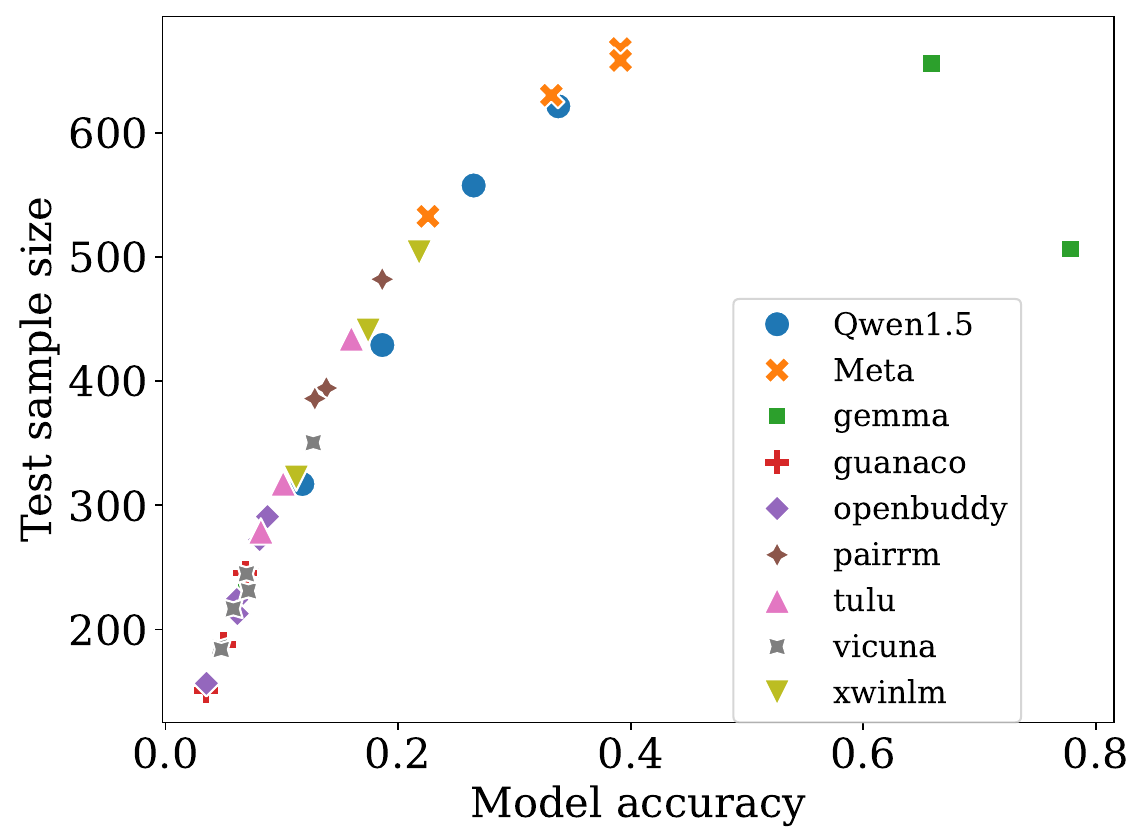}
                        \caption{Number of needed test points v.s. model accuracy when evaluating models from multiple families using \ours, with estimation error level $\epsilon=0.07$ and failure probability $\delta=0.05$.}
        \label{fig:acc}
                \end{minipage}
    \end{figure*}
    
    
    
     \vspace{-0.1in}
    \textbf{Test scaling law.} We further evaluated over 200 models on the AlpacaEval dataset to investigate potential factors affecting test sample complexity, analogous to the training scaling law~\citep{Kaplan2020ScalingLF, hoffmann2022an, bahri2024explaining}. 
    For multiple model families, \Autoref{fig:size} shows the relationship between model size and the needed test sample size for a certifiable evaluation within error level $\epsilon=0.06$ when using \ours. At first glance, it seems a larger model requires more test points, resembling the training scaling law. However, \Autoref{fig:acc} shows that the this trend is spurious:  larger models require more test points on the AlpacaEval dataset because their accuracy is closer to 0.5, leading to higher loss variance. 
    
    This finding suggests an intriguing connection between model performance and test sample complexity, offering insights into leveraging the training scaling law. Suppose an LLM continues to scale and achieves higher accuracy (above 0.5), fewer test points are sufficient for evaluation. In other words, despite increasing model sizes, we may not need a growing test dataset for LLM evaluation with a fixed estimation error level and confidence level. It therefore implies an encouraging prospect for future LLM development and evaluation.

    
\vspace{-0.05in}

In summary, \ours is recommended as the default online algorithm for certifiable and cost-efficient LLM evaluation. It effectively adapts to different datasets and models to minimize the evaluation cost and offers high flexibility in evaluation goals.

\vspace{-0.05in}

\section{Conclusion and Further Discussions}\label{sec:con}
    In this work, we propose an online evaluation framework to assess LLM performance, allowing users to determine their desired evaluation error and confidence level. This approach enables a certified and efficient evaluation process, where users can stop the evaluation once their goal is met, or continue collecting more test points if the current dataset is insufficient for a confident conclusion. Our proposed algorithm, \ours, effectively reduces the number of required test points by adapting to each model and dataset of interest.

    There are three promising directions for future work. First, instead of evaluating a particular model, our framework can be extended to compare the relative performance among multiple models, such as providing a certified ranking. The required estimation error level for distinguishing between models naturally depends on their performance differences, making our online evaluation process a well-suited approach for cost-efficient ranking. Second, an attractive extension is evaluating a model's out-of-distribution generalization performance. It is worth studying the certified evaluation when the collected dataset differs from the true underlying data distribution. Third, simultaneous evaluation across multiple tasks presents another opportunity for improvement. Evaluation tasks are known to be often correlated. Exploiting these relationships could further enhance evaluation efficiency.

\bibliography{ref}
\bibliographystyle{plainnat}

\newpage
\appendix
\section{Missing proofs}
    \textbf{General notations.} We will use $R$ for short of $R(f,P_{XY})$ when there is no ambiguity. Let $X_i, Y_i, i=1,\dots,n$ be IID sampled from $P_{XY}$, and $Z_i = \ell(f(X_i), Y_i)$. Then $Z_i$'s are IID with $0\leq Z_i \leq 1$, $\E(Z_i) = R(f,P_{XY})$, and $\var(Z_i) = \var(Z_i-0.5) \leq \E\{(Z_i-0.5)^2\} \leq 1/4$. Let $\hat{R}_n = n^{-1}\sum_{i=1}^n Z_i$. The online evaluation process will produce a sequence of estimate $\hat{R}_n$ and confidence interval (CI), characterized by its radius $\epsilon_n$, where $n$ is the number of evaluated test points so far.
    
    \textbf{Proof of \Autoref{thm:no}}
    \begin{proof}
        We first prove the upper bound by showing that \Autoref{alg:seq}, denoted as $\alggen$, is a $(n',\epsilon,\delta)$-certified test algorithm with $n'\leq 12\diff^{-2}\{\ln(1/\delta)+\ln\ln(1/\diff)\} $.
        Recall that $N$ is the number of test points evaluated when the algorithm terminates. The probability that $\alggen$ yields a wrong claim is bounded by Lemma~\ref{lemma:adap_h}, an adaptive Hoeffding-type inequality, as follows:
        \begin{align*}
            \P(\alggen \text{ makes a wrong decision}) = \P( \hat{R}_N -\epsilon_N \geq C) \leq \P(\{\exists n, \hat{R}_n - R \geq \epsilon_n\})
            \leq \delta/2. \numberthis \label{eq:thm_no_up_p3}
        \end{align*}
        As for the running time, let $n' = \frac{12\ln(4/\delta)+12\ln\ln(1/\diff)}{\diff^{2}}$. We can check that for any sufficiently small $\diff$, 
        \begin{align*}
            \epsilon_{n'}^2 = \frac{2\ln(\log(n')+1)+\ln(4/\delta)}{n'} \leq \diff^2/4.
        \end{align*}
        Therefore, the probability that the algorithm does not terminate after drawing $n'$ samples is
        \begin{align*}
            \P(\nalg \geq n') \leq \P(\hat{R}_{n'} + \epsilon_{n'} \geq C) \leq \P(\hat{R}_{n'} - R  \geq \epsilon_{n'}) \leq \delta/2, \numberthis \label{eq:thm_no_up_p4}
        \end{align*}
        where the last step is due to \Autoref{eq:thm_no_up_p3}. Combining \Autoref{eq:thm_no_up_p3, eq:thm_no_up_p4} proves that \Autoref{alg:seq} is a $(n',\epsilon,\delta)$-certified test algorithm.

        Now, we turn to prove the lower bound. Recall that $\epsilon$ can be implicitly determined by $R$, such as in the second estimation goal introduced in \Autoref{sec:formulation}. In particular, we have $\epsilon=\abs{R-C}$ when the user want to determine whether $R$ is above a threshold $C$. In this case, 
        \citet{farrell1964asymptotic} proved that for any algorithm that guarantees a $\delta$ failure probability, we have
        \begin{align*}
            \limsup_{\diff \to 0} \E \nalg \geq O\biggl(\frac{\ln\ln(1/\diff)}{\diff^{2}} \biggr). \numberthis \label{eq:thm_no_lb_p1}
        \end{align*}
        It is also known that for an easier problem where $\diff$ is known in prior~\citep{mannor2004sample}, the expected test points for any $\delta$-certified algorithm satisfies
        \begin{align*}
            \lim_{\diff \to 0, \delta \to 0} \E \nalg \geq O\biggl(\frac{\ln(1/\delta)}{\diff^{2}} \biggr). \numberthis \label{eq:thm_no_lb_p2}
        \end{align*}

        Suppose there exists an algorithm that is $(n', \epsilon,\delta)$-certified, where $n'$ satisfies
        \begin{align*}
            \lim_{\diff,\delta \to 0} n'(C,\delta,f,P_{XY})  \frac{\diff^{2}}{\ln(1/\delta)+\ln\ln(1/\diff)} = 0. \numberthis \label{eq:thm_no_lb_p4}
        \end{align*}
        As a result, for any fixed $\diff$, there exists a $\delta_0$ such that for any $\delta<\delta_0$, we have
        \begin{align*}
            \P\biggl(\nalg > \frac{\ln(1/\delta)+\ln\ln(1/\diff)}{\diff^{2}}\biggr) \leq \delta.
        \end{align*}
        Moreover, there exists an integer $n_0=\frac{\ln(1/\delta_0)+\ln\ln(1/\diff)}{\diff^{2}}$ such that for any $n>n_0$, we have
        \begin{align*}
            \P(\nalg > n) \leq \exp\{-n\diff^2+\ln\ln(1/\diff)\}.
        \end{align*}
        Then, the following holds for any $n'>n_0$:
        \begin{align*}
            \E \nalg &\leq n' + \sum_{n=n'}^{\infty} \P(\nalg > n)
            \leq n' + \int_{n'}^{\infty} \exp\{-x\diff^2+\ln\ln(1/\diff)\ dx
             \leq n' + \frac{\delta}{\diff^2}. \numberthis \label{eq:thm_no_lb_p3}
        \end{align*}

        Finally, comparing \Autoref{eq:thm_no_lb_p3} to \Autoref{eq:thm_no_lb_p1, eq:thm_no_lb_p2} yields that 
        \begin{align*}
            n' \geq O\biggl(\frac{\ln(1/\delta)+\ln\ln(1/\diff)}{\diff^{2}} \biggr),
        \end{align*}
        which contradicts with \Autoref{eq:thm_no_lb_p4}. We thus completes the proof.
    \end{proof}

    \textbf{Proof of \Autoref{thm:prune}}
    \begin{proof}
        First, \Autoref{thm:no} shows that for $\alggen$, the needed test sample size $n$ satisfies that
        \begin{align*}
             n = O\biggl(\frac{\ln(1/\delta)+2\ln\ln(1/\diff)}{\diff^{2}} \biggr) \numberthis \label{eq:prune_no_asymp}
        \end{align*}

        We prove that \Autoref{alg:adap_var} is a $(n',\epsilon,\delta)$-certified test algorithm, where $n'$ will be specified below. In particular, the known benign partition will be used as the partition subroutine in the algorithm input.
        Let $\mu_k \defeq \P(A_i)$ be the probability mass of area $A_k$, for $k=1,\dots,K$. Without loss of generality, we assume that $\mu_i$ is known. Otherwise, we can keeping drawing data points (without evaluating them) until the estimation of $\mu_i$ is sufficiently accurate.
        When a dataset is given instead of the data distribution $P_{X_Y}$, we have $\mu_k = N_k/n$, where $N_k$ is the number of test points in $A_k$. 

        
        When \Autoref{alg:adap_var} terminates and evaluate $\nalg$ points, we have 
        \begin{align*}
            \hat{R}_{\nalg} - R = \sum_{k=1}^K \mu_k(\hat{R}_k-R_k),
        \end{align*}
        where $R_k = \E_{A_k} \ell(f(X),Y)$ is the prediction error on the area $A_k$, and $\hat{R}_k = n_k^{-1} \sum_{i=1}^{n_k} Z_i^k$ is the empirical loss.  
        
        The empirical variance of the loss on $A_k$ is 
        \begin{align*}
            \hat{v}_k \defeq n_k^{-1} \sum_{i=1}^{n_k} (Z_i^k-\hat{R}_k)^2 = n_k^{-1} \sum_{i=1}^{n_k} (Z_i^k-R_k)^2 - (R^k-\hat{R}_k)^2.
        \end{align*}
        Let $\eta_k = \sqrt{\{2\ln(\log(n_k)+1)+\ln(16K/\delta)\}/n_k}$, \Autoref{lemma:adap_h} implies
        \begin{align*}
            \P\biggl(\biggl\lvert n_k^{-1} \sum_{i=1}^{n_k} (Z_i^k-R_k)^2 -v_k\biggr\rvert \leq \eta_k\biggr) &\geq 1-\delta/(4K),
            \\\P( \abs{R^k-\hat{R}_k} \leq \eta_k) &\geq 1-\delta/(4K).
        \end{align*}
        Let $\mathcal{E}=\{ v_k \leq \hat{v}_k + \eta_k + \eta_k^2, \hat{v}_k \leq v_k + \eta_k, \abs{R^k-\hat{R}_k} \leq \eta_k\}$, we have $\P(\mathcal{E}) \geq 1-\delta/2$.  
        Evoking Lemma~\ref{lemma:adap_b} on event $\mathcal{E}$, we have
        \begin{align*}
            \P\biggl(\hat{R}_k-R_k \geq \epsilon_k, \mathcal{E} \biggr) \leq \delta/(8K),
        \end{align*}
        where 
        \begin{align*}
            \epsilon_k = 2\eta_k^2/3+2\sqrt{(\hat{v}_k+\eta_k+\eta_k^2)\eta_k^2}.
        \end{align*}
        As a result, a union bound gives
        \begin{align*}
            \P(\mathcal{A} \text{ makes a wrong decision}) &\leq 1-\P(\mathcal{E}) + \P( \hat{R}_{\nalg} - R \geq \sum_k \mu_k \epsilon_{k}, \mathcal{E}) 
            \\&\leq \delta/2 + \sum_k\P(\exists k, \hat{R}_{k} - R_k \geq  \epsilon_{k}) 
            \leq 5\delta/8. \numberthis\label{eq:thm_var_p1}
        \end{align*}

        Regarding the required sample complexity, let $n'=n\ln(K+1)\sum_k \max\{v_k, \diff^{2/3}\}$. 
        Suppose the algorithm does not terminate after evaluating $n'$ points.
        Note that 
        \begin{align*}
            n_i/\{\hat{v}_i,\eta_i\} &= n_j/\max\{\hat{v}_j,\eta_j\}, 1\leq i \leq j \leq K,
             \sum_k n_k =n'. \numberthis \label{eq:thm_var_n}
        \end{align*}
        Therefore, $n_k = n\ln(K+1)\frac{\sum_k \max\{v_k, \diff^{2/3}\}}{\sum_k \max\{\hat{v}_k,\eta_k\}} \max\{\hat{v}_k,\eta_k\}$.
        Now, we can check that  
        \begin{align*}
            n_k = O(n\ln(K+1) \max\{v_k, \diff^{2/3}\}).
        \end{align*}
        To see it, when $v_k\geq O(\diff^{2/3})$, the above $n_k$ ensures $ \eta_k \leq O(\diff^{2/3})$, implying that $\max\{\hat{v}_k,\eta_k\} = \max\{\hat{v}_k,\diff^{2/3}\}$. 
        Similarly, when $v_k\leq O(\diff^{2/3})$, $n_k = O(n\ln(K+1) \diff^{2/3})$ ensures $\eta_k = O(\diff^{2/3})$ and therefore $\max\{\hat{v}_k,\eta_k\} = \max\{\hat{v}_k,\diff^{2/3}\}$.
        As a result, this $n_k$ is the solution of \Autoref{eq:thm_var_n}.
        
        Finally, on the event $\mathcal{E}$, the above $n'$, up to a universal constant, satisfies that $\epsilon_k \leq \diff/2$,
        implying that $\P(\nalg \geq n', \mathcal{E}) \leq \delta/8$. 
        Thus, \Autoref{alg:adap_var} is a $(n',\epsilon,\delta)$-certified algorithm and we complete the proof. 
    \end{proof}

    \textbf{Proof of \Autoref{coro:superG}}
    \begin{proof}
        Recall that the loss is bounded in $[0,1]$ by \Autoref{asmp:bound}. Given a super-Gaussian loss distribution, the probability mass of $A_k$ is 
        \begin{align*}
            \mu_k &= \P(A_k) \geq A\int_{(k-1)/K}^{k/K} \exp(-z^2/\sigma^2)dz \geq 
            A\int_{(k-1)/K}^{k/K} \exp(-z/\sigma^2)dz
            \\ &= A\sigma^2 e^{-k/(K\sigma^2)}(e^{1/(K\sigma^2)}-1) \geq \frac{A}{K}e^{-k/(K\sigma^2)}.
        \end{align*}
        For an equally-spaced partition, we have $v_k \leq 1/K^2$. 
        Therefore, when $K\leq \diff^{-1/3}$, \Autoref{asmp:loss_dist} is satisfied if 
        \begin{align*}
            \mu_k \geq \ln(K+1)/K^2, k=1,\dots,K
        \end{align*}
        which is equivalent to
        \begin{align*}
            K/\ln(K+1) \geq \exp\{1/\sigma^2\}/A. \numberthis \label{eq:thm_super_K}
        \end{align*}
        Therefore, when $K\leq \diff^{-1/3}$ and \Autoref{eq:thm_super_K} holds, \Autoref{thm:prune} implies
        \begin{align*}
            \rho = O\biggl(\ln(K+1) \sum_{k=1}^K \max\{v_k, \diff^{2/3}\}\biggr) = O(\ln(K+1)/K),
        \end{align*}
        thus completes the proof.
    \end{proof}

\section{Technical Lemmas}

    \begin{lemma}[Hoeffding Inequality \protect{\citep[][Theorem 2.8]{boucheron2013concentration}}]\label{lemma:hoeffding}
        For independent observations $X_1, \dots, X_n$ such that $a_i \leq X_i \leq b_i$ almost surely (a.s.), let $S_n = \sum_{i=1}^n \{X_i-\E(X_i)\}$, we have that
        \begin{align*}
            \P(S_n \geq t) \leq \exp\biggl\{\frac{-2t^2}{\sum_{i=1}^n(b_i-a_i)^2}\biggr\}.
        \end{align*}
    \end{lemma}

    \begin{lemma}[Bernstein Inequality \protect{\citep[][Equation 2.10]{boucheron2013concentration}}]\label{lemma:bernstein}
        Assume independent observations $X_1, \dots, X_n$ such that $  X_i \leq b$ a.s.. Let $S_n = \sum_{i=1}^n \{X_i-\E(X_i)\}$ and $v_n = \sum_{i=1}^n\var(X_i)$, we have
        \begin{align*}
            \P(S_n \geq t) \leq \exp\biggl\{\frac{-t^2}{2(v_n+bt/3)}\biggr\},
        \end{align*}
        or equivalently,
        \begin{align*}
            \P\biggl(S_n \geq \frac{b\ln(1/\delta)}{3}+\frac{1}{3}\sqrt{b^2\ln^2(1/\delta)+18v_n\ln(1/\delta)}\biggr) \leq \delta.
        \end{align*}
    \end{lemma}

    \begin{lemma}[Adaptive Hoeffding Inequality \protect{\citep[][Theorem 1]{zhao2016adaptive}}]\label{lemma:adap_h}
        Let $\epsilon_n = \sqrt{\frac{2\ln(\log(n)+1)+\ln(4/\delta)}{n}}$.
        For independent observations $X_1, \dots, X_n$ such that $0 \leq X_i \leq 1$ a.s., let $S_n = n^{-1}\sum_{i=1}^n \{X_i-\E(X_i)\}$.
        Then, we have
        \begin{align*}
            \P(\{\exists n, S_n/n \geq \epsilon_n \}) \leq \delta/2. 
        \end{align*}
    \end{lemma}

    \begin{lemma}[Maximal Form of Bernstein Inequality \protect{\citep{kevei2011note}}]\label{lemma:bern_max}
        Assume independent observations $X_1, \dots, X_n$ such that $  X_i \leq b$ a.s.. Let $S_n = \sum_{i=1}^n \{X_i-\E(X_i)\}$ and $v_n = \sum_{i=1}^n\var(X_i)$, we have
        \begin{align*}
            \P(\max_{1\leq i\leq n}S_i \geq t) \leq \exp\biggl\{\frac{-t^2}{2(v_n+bt/3)}\biggr\},
        \end{align*}
        or equivalently,
        \begin{align*}
            \P\biggl(\max_{1\leq i\leq n} S_i \geq \frac{b\ln(1/\delta)}{3}+\frac{1}{3}\sqrt{b^2\ln^2(1/\delta)+18v_n\ln(1/\delta)}\biggr) \leq \delta.
        \end{align*}
    \end{lemma}

    \begin{lemma}[Adaptive Bernstein Inequality ]\label{lemma:adap_b}
        Let $u_n=2\ln(\log(n)+1)+\ln(4/\delta)$ and $\epsilon_n = (b u_n+\sqrt{b^2u_n^2+18v_{2n} u_n})/(3n)$.
        For independent observations $X_1, \dots, X_n$ such that $0 \leq X_i \leq 1$ a.s., let $S_n = \sum_{i=1}^n \{X_i-\E(X_i)\}$.
        Then, we have
        \begin{align*}
            \P(\{\exists n, S_n/n \geq \epsilon_n \}) \leq \delta/2. 
        \end{align*}
    \end{lemma}
    \begin{proof}
        Applying Lemma~\ref{lemma:bern_max} yields 
        \begin{align*}
            \P(\{\exists n, S_n/n \geq \epsilon_n \})  
            &= \P(\cup_{n=1}^\infty\{S_n \geq n\epsilon_n\} )
            \\& = \P(\cup_{l=0}^\infty\cup_{2^l \leq n \leq 2^{l+1}}\{S_n \geq n\epsilon_n\} )
            \\&\leq \P(\cup_{l=0}^\infty\{\max_{2^l \leq n \leq 2^{l+1}} S_n \geq 2^l\epsilon_{2^l}\} )
            \\&\leq \sum_{l=0}^\infty \P\biggl[ \max_{1\leq n \leq 2^{l+1}} S_n \geq 2^l\epsilon_{2^l} \biggr]
            \\& \leq \sum_{l=0}^\infty e^{-u_{2^l}} = \sum_{l=0}^\infty (l+1)^{-2}\delta/4 \leq \delta/2,
        \end{align*}
        thus completes the proof.
    \end{proof}



\section{Missing Algorithms}
   \begin{algorithm}
        \caption{Vanilla Online Evaluation (Seq)}\label{alg:seq}
        \begin{algorithmic}[1]
            \Require The estimation error level $\epsilon$, failure probability $\delta$, and $P_{XY}$
            \For{Round $n=1,2,\dots$}
                \State Let $\epsilon_n \leftarrow \sqrt{\frac{2\ln(\log(n)+1)+\ln(4/\delta)}{n}}$
                \State Random sample $(X_n, Y_n)$ from $P_{XY}$ \Comment{When a test dataset $D$ is given instead of $P_{XY}$, sample the next point in $D$} 
                \State Evaluate the loss $Z_n \leftarrow \ell(f(X_n),Y_n)$
                \State Let $\hat{R}_n \leftarrow n^{-1}\sum_{i=1}^{n} Z_i$
                \If{$\epsilon_n \leq \epsilon$}
                    \State Terminate and return $\hat{R}_n, \epsilon_n$
                \EndIf
            \EndFor
            \Ensure The estimated loss $\hat{R}_n$, confidence interval radius $\epsilon_n$, and number of evaluated points $n$
        \end{algorithmic}
    \end{algorithm}    

\section{Experiment Details and Further Experiments}

\textbf{Compute Resources}

\textbf{Ablation study for the influence on embedding models}

We consider three embedding models: a pre-trained BERT, or text-embedding-3-large and text-embedding-ada-002 model from OpenAI.
With the same setting as the real-world experiments performed in \Autoref{sec:exp}, we report the saving ratios for evaluating four models on three datasets in \Autoref{fig:mmlu_ratio, fig:alpacaeval_ratio, fig:math_ratio}. 
Also, the empirical failure probability is zero for all settings. We find that using different embedding models leads to highly similar results. Nevertheless, it does not imply that embedding model is unimportant. In contrary, the efficiency of \ours heavily depends on the goodness of partition. This is confirmed by our simulation studies, where \ours performs better in the easy-to-distinguish scenario. For these datasets, we conjecture that the proposed partition method, \Autoref{alg:part}, does not extract enough information from the embedding vectors. We anticipate a higher saving ratio if more informative embeddings of the queries are available, or a more effective partition algorithm is deployed. One potential direction is to use the intrinsic attributes of queries as embedding vectors, such as the one used in \citep{polo2025tiny}. 

\begin{figure}[tb]
    \centering
    \includegraphics[width=0.9\linewidth]{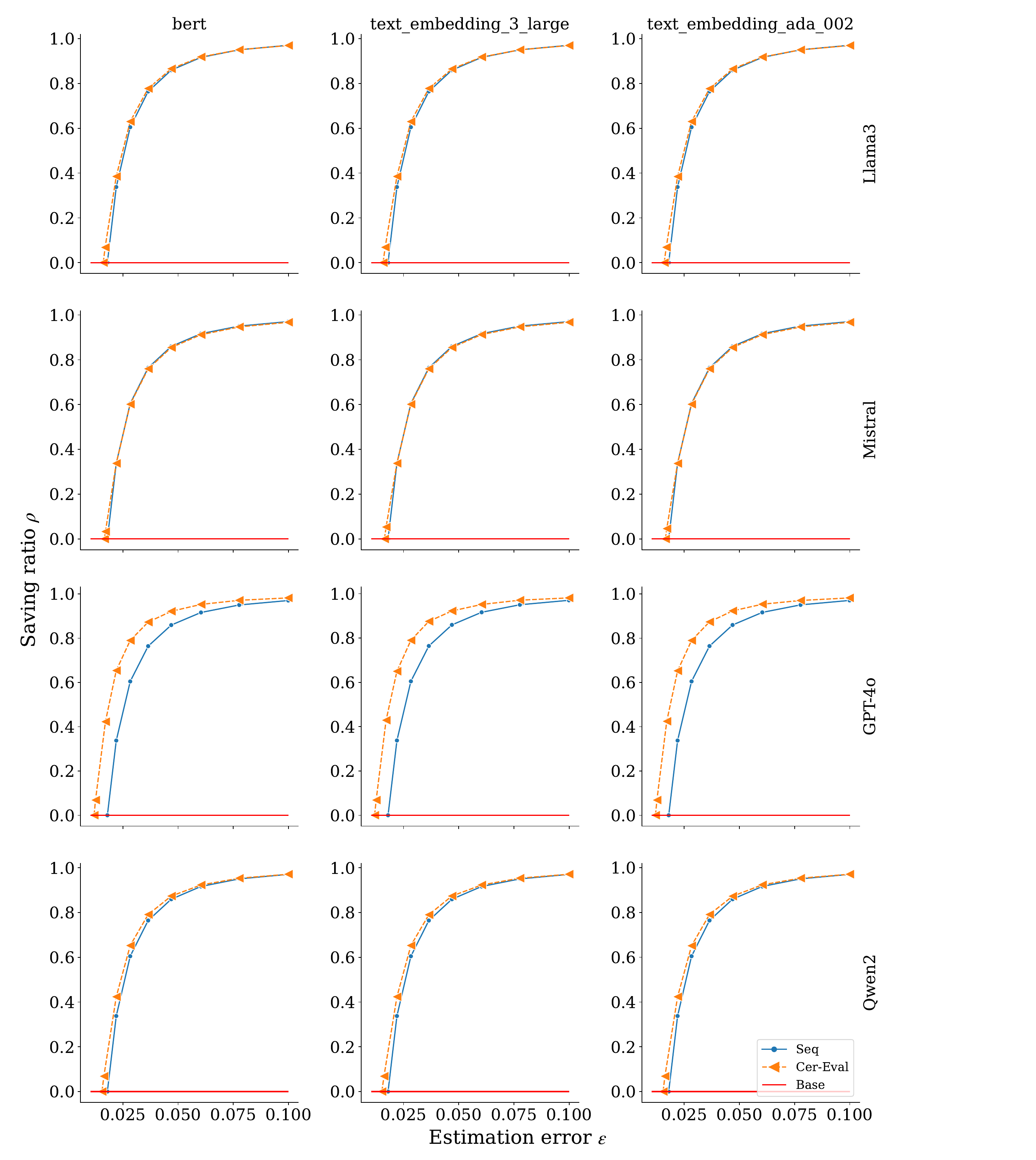}
    \caption{Percentage of test points saved by the proposed algorithms compared to Base when evaluating models on the MMLU dataset. }
    \label{fig:mmlu_ratio}
\end{figure}

\begin{figure}[tb]
    \centering
    \includegraphics[width=0.9\linewidth]{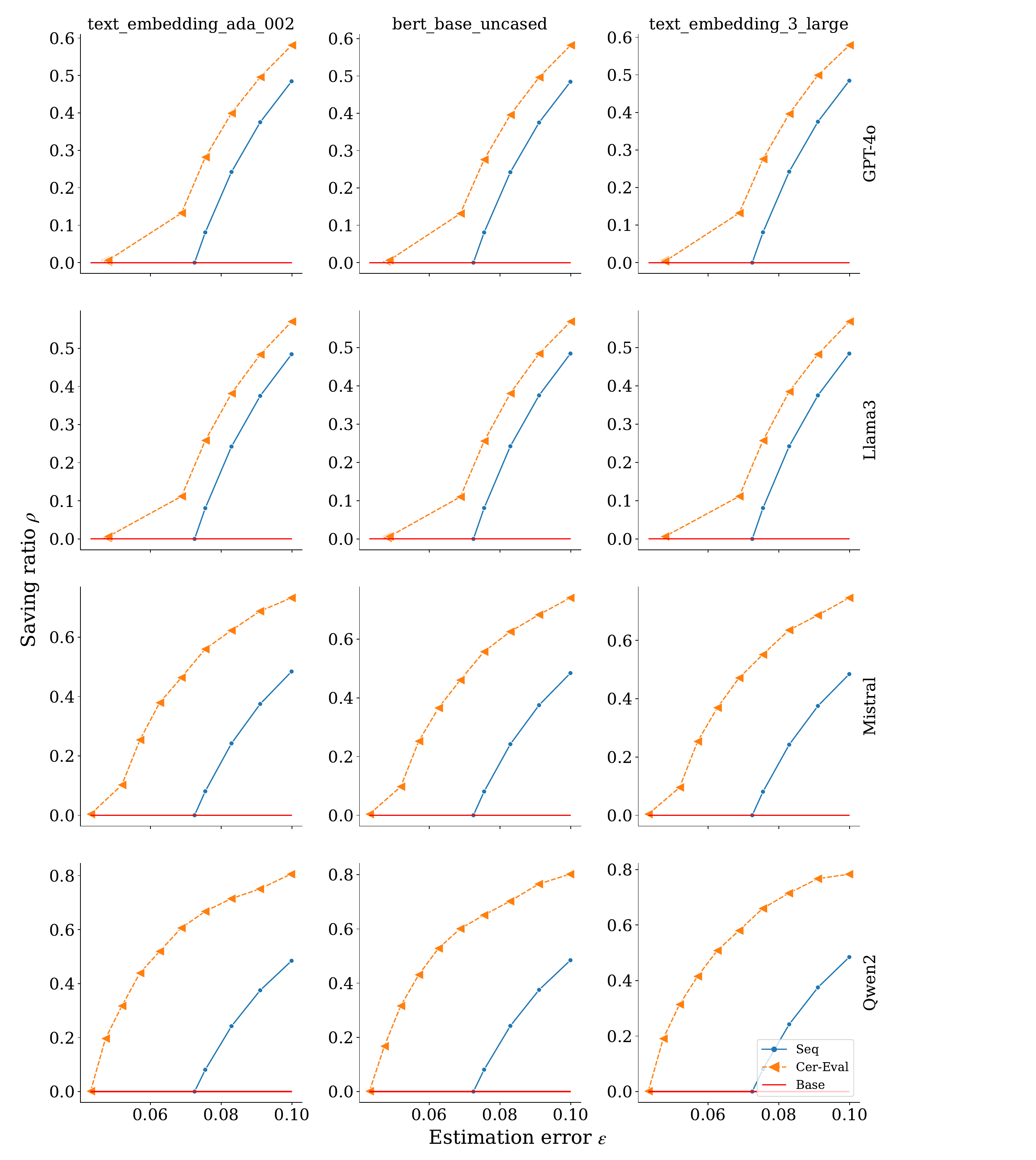}
    \caption{Percentage of test points saved by the proposed algorithms compared to Base when evaluating models on the AlpacaEval dataset. }
    \label{fig:alpacaeval_ratio}
\end{figure}

\begin{figure}[tb]
    \centering
    \includegraphics[width=0.9\linewidth]{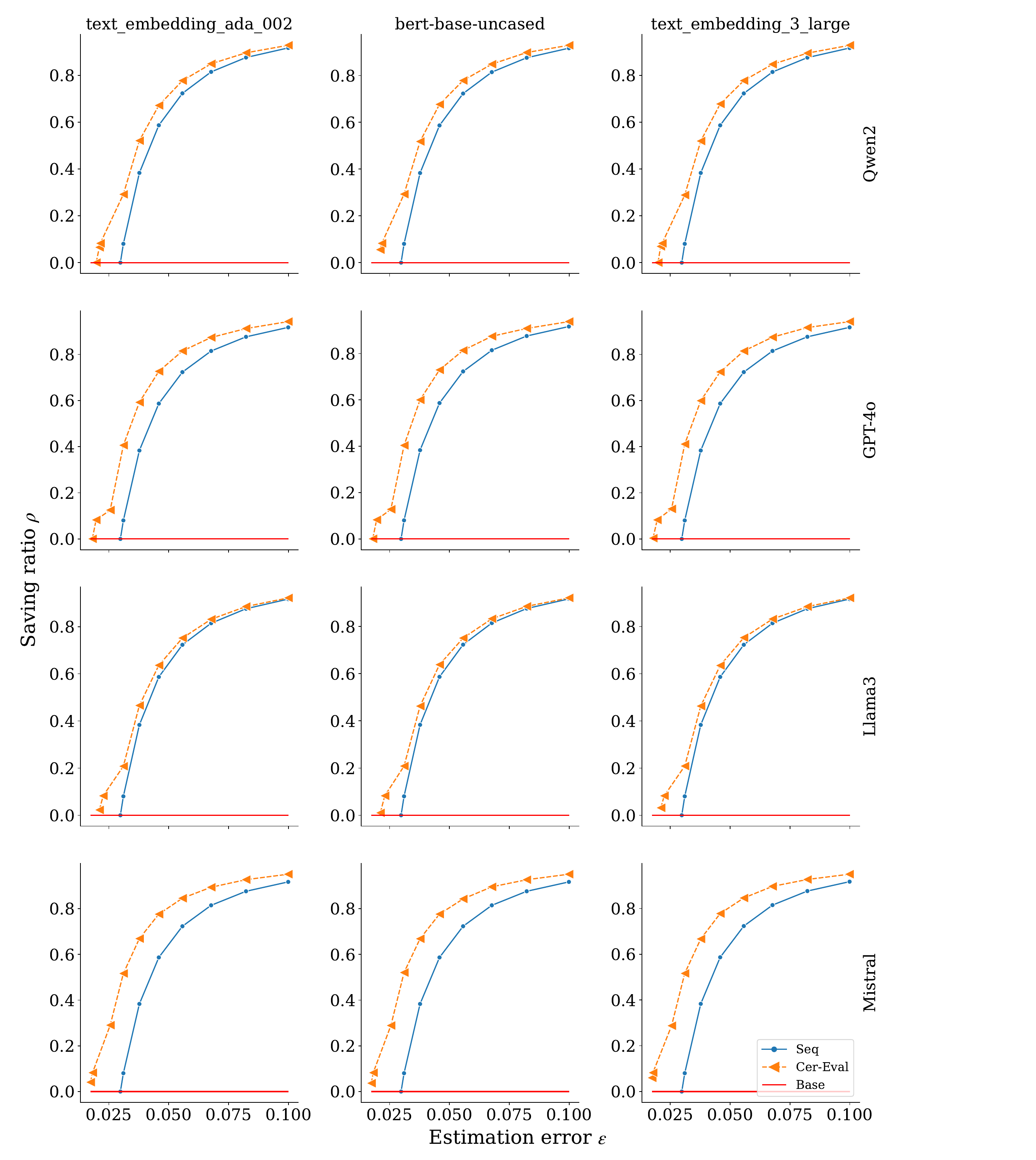}
    \caption{Percentage of test points saved by the proposed algorithms compared to Base when evaluating models on the MATH dataset. }
    \label{fig:math_ratio}
\end{figure}

\textbf{Additional Models to Verify Test Scaling Law}
    We plot the accuracy v.s. test sample size curve for over 200 models on MATH dataset, as shown in \Autoref{fig:test_scale}
    \begin{figure}[tb]
        \centering
        \includegraphics[width=0.5\linewidth]{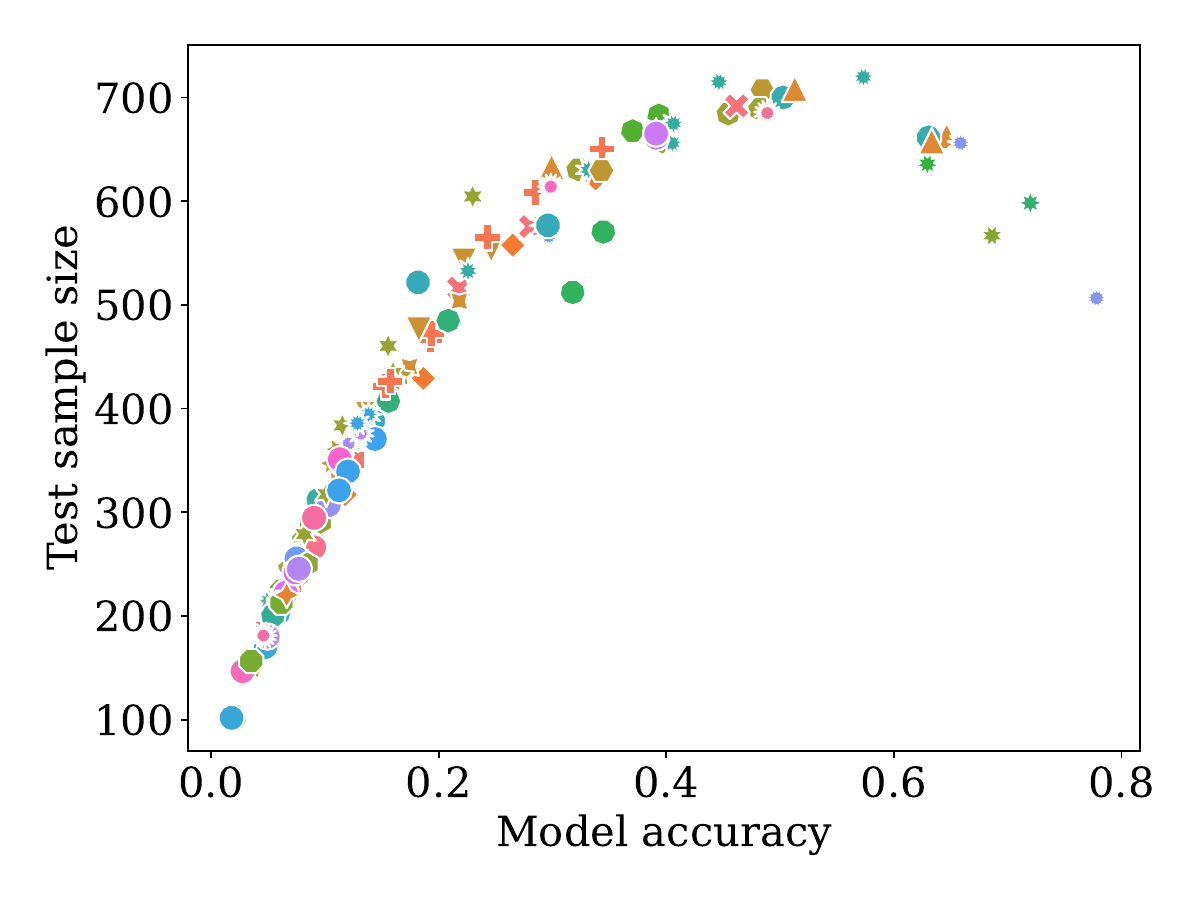}
        \caption{Number of needed test points v.s. model accuracy when evaluating models from multiple families using \ours, with estimation error level $\epsilon=0.07$ and failure probability $\delta=0.05$.}
        \label{fig:test_scale}
    \end{figure}

\section{Impact Statement}
    Our work enables a certifiable and efficient assessment of various aspects of LLMs, including their capabilities, robustness to adversarial prompts, and alignment with human values. 
    We therefore anticipate a positive societal impact, as timely and rigorous evaluation plays a crucial role in enhancing AI responsibility, mitigating potential risks, and ensuring that AI technologies align with ethical and safety standards.

 \FloatBarrier

\end{document}